\documentclass{article}

\PassOptionsToPackage{numbers, compress}{natbib}
 \usepackage[preprint]{neurips_2026}

\usepackage[utf8]{inputenc} 
\usepackage[T1]{fontenc}    
\usepackage{hyperref}       
\usepackage{url}            
\usepackage{booktabs}       
\usepackage{amsfonts}       
\usepackage{nicefrac}       
\usepackage{microtype}      
\usepackage{xcolor}         
\usepackage{amsmath}
\usepackage{algorithm}
\usepackage{algorithmic}
\usepackage{graphicx}
\usepackage{amsthm}
\usepackage{enumitem}

\newtheorem{assumption}{Assumption}
\newtheorem{definition}{Definition}
\newtheorem{proposition}{Proposition}
\newtheorem{lemma}{Lemma}
\newtheorem{theorem}{Theorem}
\newtheorem{corollary}{Corollary}
\newtheorem{observation}{Observation}
\newtheorem{remark}{Remark}

\title{Mitigating Retaliatory Algorithmic Collusion in Repeated Games}

\author{%
  Karthik ~Sivachandran\\
  Department of Computer Science\\
  Purdue University\\
  \texttt{ksivacha@purdue.edu} \\
  \And
    Rohan R. ~Paleja \\
    Department of Computer Science \\
    Purdue University \\
   \texttt{rpaleja@purdue.edu} \\
}

\begin{document}

\maketitle

\begin{abstract}
Reinforcement learning agents trained to maximize their own reward in repeated interactions can converge to supra-competitive outcomes resembling explicit collusion, without communication or shared design. Existing mitigation approaches are largely tied to specific economic settings, like two-sided platforms and auctions, leaving open how to design interventions for general repeated games. We address this gap by formalizing the connection between empirical observations from prior work on Q-learning collusion and classical theory of Simple Penal Codes (SPCs). We show any non-trivial SPC induces a quantifiable conditional dependence in agents' policies, detectable via the total variation distance between an agent's action distributions across cooperation and defection histories. Building on this connection, we propose CURB (Collusion Unwinding via Reward shaping and Belief injection), a reward-shaping framework that penalizes this Total Variation (TV) distance signal during Q-learning and is guaranteed to convert any SPC fixed point of the dynamics into a trivial one, thus precluding collusive equilibria  sustained by punishment threats. Empirically, CURB substantially reduces collusion by Q-learning agents in both Bertrand and Cournot Competition Repeated Games. We further demonstrate that CURB extends to deep Q-network agents in Bertrand competition, suggesting the mechanism generalizes beyond tabular Q-learning.
\end{abstract}

\section{Introduction}
\label{sec:intro}

In August 2024, the U.S.\ Department of Justice sued RealPage, alleging that its pricing software had helped landlords coordinate rent increases across millions of American apartments, without any landlord ever explicitly agreeing to fix prices~\citep{doj2024realpage}. The mechanism was algorithmic: pricing systems, acting on profit signals, learned a behavior pattern that in aggregate looked like a cartel. Millions of renters had to pay supra-competitive rents. This is one of several recent cases highlighting a phenomenon documented across diverse multi-agent settings: reinforcement learning agents trained to maximize their own reward in repeated interactions can converge to supra-competitive outcomes resembling collusion, without communication or shared design. The phenomenon has been observed in pricing oligopolies~\citep{calvano2020collusion, klein2021}, ad auctions~\citep{banchio2022auction}, and richer multi-agent learning environments~\citep{leibo2017multiagent, hughes2018}. Detecting and mitigating this behavior has become an active concern at the intersection of economics, computer science, and antitrust policy~\citep{calvano2019policy, harrington2018,calvano2020policy}.

Despite the multi-domain nature of algorithmic collusion, existing mitigation frameworks are limited. They remain tied to specific settings or operate on learning hyperparameters such as exploration rate or discount factor (Section~\ref{sec:related}). We address a deeper question: what equilibrium object sustains  the reward–punishment dynamics observed in algorithmic collusion across repeated games, and how can interventions provably dismantle it?
 We address this question by formalizing the connection between algorithmic collusion and the classical theory of repeated games. Our approach builds on a key observation from ~\citet{calvano2020collusion}: Q-learning agents that converge to collusive outcomes implicitly implement \emph{Simple Penal Codes (SPCs)} \citep{Abreu_1988}---reward-punishment strategies in which deviations trigger retaliatory paths. ~\citet{calvano2020collusion} noted this empirical resemblance, but the connection to Abreu's theory \citep{Abreu_1988} has not been formalized or used for intervention. We make this bridge explicit to derive a mitigation framework. We make three contributions:
\textbf{1)} We formalize the connection between Calvano's empirical observation and Abreu's SPC theory, proving (Proposition~\ref{prop:tv-detect}) that any non-trivial SPC induces a quantifiable conditional dependence in agents' policies, detectable via a total variation distance bound $\varepsilon^*(\delta) > 0$ that depends only on game primitives. \textbf{2)} We propose CURB (Collusion Unwinding via Reward shaping and Belief injection), a framework that penalizes the TV-distance signal during Q-learning. We prove (Theorem~\ref{thm:main}) that for sufficient penalty strength, every SPC fixed point of shaped Q-learning is \emph{trivial}, thus preventing collusive equilibria sustained by punishment threats. \textbf{3)} We empirically validate CURB across two repeated-game settings (Bertrand and Cournot competition) and two learning algorithms (tabular Q-learning and DQN), demonstrating that collusion is heavily mitigated.

\vspace{-3mm}
\section{Related Work}
\label{sec:related}
\vspace{-3mm}
\textbf{Observation of algorithmic collusion.}
A substantial body of work has documented algorithmic collusion across diverse settings.~\citet{bertrand2025} show that self-play Q-learners can collude in the iterated prisoner’s dilemma, while~\citet{Hansen2021} and~\citet{asker2022} demonstrate that even simple pricing algorithms can generate supra-competitive outcomes in standard pricing games. In oligopoly settings,~\citet{calvano2020collusion} and~\citet{klein2021} show that Q-learning agents can sustain collusion through reward–punishment dynamics in Bertrand competition, and~\citet{Hettich2021} shows that deep Q-networks~\citep{mnih2013} converge to such outcomes more rapidly than tabular learners. Similar behavior has been observed in Cournot oligopoly~\citep{WALTMAN20083275}, dealer markets~\citep{cont2024}, and auction environments~\citep{banchio2022auction}. More broadly, cooperation among learning agents arises naturally in sequential social dilemmas~\citep{leibo2017multiagent, hughes2018}; while beneficial in some settings, these dynamics become problematic when coordination harms external parties. Recent work further shows that LLM-based agents can exhibit collusive behavior in strategic environments~\citep{fish2024,lin2024}. Real-market evidence confirms these risks are not merely theoretical:~\citet{assad2024} document an increase in duopoly margins where both stations adopt algorithmic pricing in Germany's retail gasoline market, with similar findings in online retailers of over-the-counter pharmaceuticals~\citep{brown2023} and on e-commerce marketplaces~\citep{musolff2022}.

\textbf{Mitigation approaches.} ~\citet{calvano2021} demonstrate that imperfect monitoring of competitors' actions modestly weakens but does not eliminate collusion.~\citet{banchio2022auction} disrupt collusion in ad auctions by switching first-price to second-price, an intervention tied to auction format. Platform-side approaches, learned policies via Stackelberg POMDPs~\citep{brero2022} and the platform pricing rules of~\citet{johnson2023}, require a centralized platform with authority to modify payoffs or rules and do not generalize across repeated games.~\citet{hartline2024,hartline2025} propose ex-post auditing, which detects collusion but does not prevent it.

\textbf{Positioning of CURB.}
CURB targets the punishment-threat equilibrium structure that sustains collusion rather than the conditions surrounding learning, the infrastructure around the agents, or the symptoms collusion produces in market outcomes. Where prior approaches detect collusion after it occurs, redesign the strategic environment, or impose platform-side rules, CURB modifies the agent's own learning rule with a penalty derived from the equilibrium structure of repeated games, providing a formal guarantee that the collusive equilibrium sustained by threats do not survive as a fixed point. The mechanism requires only observing competitors' actions, which is standard in any repeated game with public actions, and applies across repeated games in the same independent-learner settings where collusion arises, without changing the strategic environment or requiring centralized infrastructure.
 

\vspace{-3mm}
\section{Preliminaries}
\vspace{-3mm}
\label{sec:preliminaries}

All theoretical results are stated for a general stage game.
Let $G = (\{A_i\}_{i=1}^n, \{r_i\}_{i=1}^n)$, where 
$N = \{1, \ldots, n\}$ is the set of players, $\mathcal{A} = \times_{i \in N} \mathcal{A}_i$ 
is the joint action space with $\mathcal{A}_i$ finite for each player $i$, and 
$r = (r_1, \ldots, r_n)$ with $r_i : \mathcal{A} \rightarrow \mathbb{R}$ is the 
profile of payoff functions. We denote a joint action at time $t$ by 
$q(t) = (q^t_1, \ldots, q^t_n) \in \mathcal{A}$, where $q^t_i$ is the action taken by player $i$ at time $t$.
The \emph{supergame} $G^{\infty}(\delta)$ is the infinitely repeated version of $G$ 
with common discount factor $\delta \in (0,1)$. A \emph{pure strategy} for player $i$ 
is a sequence of functions $\sigma_i = (\sigma_i(1), \sigma_i(2), \ldots)$ where 
$\sigma_i(1) \in \mathcal{A}_i$ and $\sigma_i(t) : \mathcal{A}^{t-1} \rightarrow 
\mathcal{A}_i$ for $t \geq 2$, mapping the history of all past joint actions to an 
action at period $t$. A \emph{strategy profile} is $\sigma = (\sigma_1, \ldots, \sigma_n) 
\in \Sigma \equiv \times_{i \in N} \Sigma_i$, where $\Sigma_i$ is the strategy set of player $i$.
A \emph{path} $Q = \{q(t)\}_{t=1}^{\infty} \in \mathcal{A}^{\infty}$ is an infinite 
sequence of joint action profiles. Every strategy profile $\sigma$ generates a unique 
path $Q(\sigma)$ defined inductively by $q(\sigma)(1) = \sigma(1)$ and 
$q(\sigma)(t) = \sigma(t)(q(\sigma)(1), \ldots, q(\sigma)(t-1))$. 
The \emph{discounted payoff} to player $i$ from path $Q$ is defined by 
%
    $v_i(Q) = \sum_{t=1}^{\infty} \delta^{t-1} r_i(q(t))$.
%

\begin{assumption}[Bounded Payoffs]
\label{ass:bounded}
There exist $r_{\min}, r_{\max} \in \mathbb{R}$ such that 
$r_i(a) \in [r_{\min}, r_{\max}]$ for all $i \in N$, $a \in \mathcal{A}$.
\end{assumption}

\begin{assumption}[Finite Action Space]
$\mathcal{A}_i$ is finite for all $i \in N$.
\end{assumption}

\begin{assumption}[Markov State]
\label{ass:ms}
The state $s_t \in \mathcal{S}$ observed by agents at period $t$ is a 
sufficient statistic of the payoff-relevant history. 
The state space $\mathcal{S}$ is finite.
\end{assumption}


\textbf{Theoretical regime.}
Our results assume:
\begin{enumerate}[noitemsep]
\item[(T1)] The cooperative path $Q^0$ is pure: $q^0_i(t) \in A_i$ is deterministic for all $i \in N$, $t \ge 1$.
\item[(T2)] Learned policies $\pi_j$ are stationary policies mapping $S \to \Delta(A_j)$, where $\Delta(A_j)$ is the set of probability distributions on $A_j$. This includes deterministic greedy and $\varepsilon$-greedy policies.
\item[(T3)] The total variation (TV) distance in the shaped reward is computed over the exact conditional action distributions induced by $\pi_j$.
\end{enumerate}
Under (T2), $\pi_j(\cdot \mid a^{t-1}_i \in D^\phi_i)$ and $\pi_j(\cdot \mid a^{t-1}_i \in C^\phi_i)$ are distributions on $A_j$, and TV $\in [0, 1]$ takes continuous values. Our intervention operates in an \emph{empirical regime} where TV is estimated from finite action-frequency windows. Empirical TV converges to theoretical TV as windows grow and exploration anneals.

\subsection{Simple Penal Codes}

We recall the central equilibrium concept from~\citet{Abreu_1988}. For definitions of Nash Equilibrium and subgame perfect equilibrium we defer to~\citet{RUBINSTEIN19791}.

\begin{definition}[Simple Strategy Profile]
\label{def:ssp}
Let $Q^0, Q^1, \ldots, Q^n \in \mathcal{A}^{\infty}$ be paths. 
The \emph{simple strategy profile} $\sigma(Q^0, Q^1, \ldots, Q^n)$ specifies:
\begin{enumerate}[noitemsep]
    \item play $Q^0$ until some player deviates singly
    \item if player $j$ deviates singly from any ongoing path, switch to $Q^j$ immediately
    \item simultaneous deviations leave the ongoing path unchanged
\end{enumerate}
\end{definition}

\begin{definition}[Simple Penal Code]
\label{def:spc}
A simple strategy profile $\sigma(Q^0, Q^1, \ldots, Q^n)$ is a 
\emph{Simple Penal Code} (SPC) if it is a subgame perfect equilibrium 
of $G^{\infty}(\delta)$.
\end{definition}

For the statement of the sustainability condition, we introduce two 
quantities. The \emph{one-shot deviation gain} for player $j$ at 
period $t$ along path $Q^i$ is given by Equation~\ref{eq:one-shotdev}.
\begin{equation}
\label{eq:one-shotdev}
\Delta_j(q_j^*, q_{-j}^i(t)) 
= r_j(q_j^*, q_{-j}^i(t)) - r_j(q^i(t))
\end{equation}

In Equation~\ref{eq:one-shotdev}, $q_j^* \in \arg\max_{a \in \mathcal{A}_j} r_j(a, q_{-j}^i(t))$ 
is player $j$'s best deviation. The \emph{continuation value} of 
player $j$ from period $t+1$ onward along path $Q$ is given by Equation~\ref{eq:cv}
\begin{equation}
\label{eq:cv}
v_j(Q;\, t+1) = \sum_{s=1}^{\infty} \delta^{s} r_j(q(t+s)).
\end{equation}

The following result characterizes when a simple strategy profile 
is an SPC.

\begin{proposition}[\citet{Abreu_1988}]
\label{prop:abreu}
Under Assumptions 1-3, the simple strategy profile 
$\sigma(Q^0, Q^1, \ldots, Q^n)$ is a subgame perfect equilibrium 
if and only if for all $j \in N$, $i \in \{0, \ldots, n\}$, 
and $t \geq 1$:
\begin{equation}
\underbrace{\Delta_j(q_j^*,\, q_{-j}^i(t))}_{\text{gain from deviating}}
\;\leq\;
\underbrace{v_j(Q^i;\,t+1) - v_j(Q^j)}_{\text{loss from punishment}}.
\label{eq:sustainability}
\end{equation}
\end{proposition}

We call Equation~\ref{eq:sustainability} the \emph{sustainability condition}. Collusion is sustained when deviation gains are outweighed by punishment losses.

\begin{remark}[One-shot deviation]
\label{rem:osdp}
Under Assumptions~\ref{ass:bounded}--\ref{ass:ms}, Proposition~\ref{prop:abreu} is equivalent to the \emph{one-shot deviation principle}: $\sigma$ is a subgame-perfect equilibrium iff no agent can profitably deviate for a single period and then revert to $\sigma$ at any subgame.
\end{remark}

\begin{remark}
We consider agents that learn via tabular Q-learning~\citep{Watkins1992}. Full details of the update rule are provided in Appendix~\ref{app:q-learning}.
In the repeated game $G^\infty(\delta)$, the state observed by agents at period $t$ is the $k$-period action history $s_t = (q(t-1), q(t-2), \ldots, q(t-k)) \in \mathcal{A}^k$.
\end{remark}


\subsection{Collusion as an Implicit Simple Penal Code}
\label{sec:prelim-collusion}
 
We use the term \emph{collusion} to refer to outcomes where agents sustain supra-Nash payoffs at the expense of a third party whose welfare falls below its Nash level.
 
\begin{definition}[Collusive Strategy Profile]
\label{def:collusive}
A strategy profile $\sigma$ is \emph{collusive} if it implements a non-trivial Simple Penal Code and $W(Q^0) < W(Q^{\mathrm{Nash}})$. $W(Q) = \sum_{t=1}^\infty \delta^{t-1} w(q(t))$ is the discounted third-party welfare under path $Q$, $w : A \to \mathbb{R}$ is the per-period third-party welfare function, and $Q^{\mathrm{Nash}}$ is the path generated by the stage-game Nash equilibrium.
\end{definition}

The trivial SPC $\sigma(Q^{\mathrm{Nash}}, \dots, Q^{\mathrm{Nash}})$ is non-collusive since $Q^0 = Q^{\mathrm{Nash}}$ implies $W(Q^0) = W(Q^{\mathrm{Nash}})$.
 
\begin{observation}
\label{obs:calvano}
When Q-learning agents converge to a collusive outcome in the sense of Definition~\ref{def:collusive}, their learned policies $(\pi_1, \dots, \pi_n)$ implicitly implement a Simple Penal Code $\sigma(Q^0, Q^1, \dots, Q^n)$, where $Q^0$ is the collusive path satisfying $W(Q^0) < W(Q^{\mathrm{Nash}})$ and each $Q^j$ is a punishment path encoding state-contingent retaliation following a unilateral deviation by player $j$.
\end{observation}
 
Observation~\ref{obs:calvano} is supported empirically by~\citet{calvano2020collusion}, who document that converged Q-learning policies in the Bertrand game exhibit reward-punishment behavior consistent with an implicit Simple Penal Code: agents cooperate on a supra-Nash path and retaliate following unilateral deviations. We treat Observation~\ref{obs:calvano} as an empirically motivated modeling assumption rather than a formally established property.

 \textbf{Intervention rationale.} Under Obs.~\ref{obs:calvano}, preventing any collusive SPC (Def.~\ref{def:collusive}) from being a fixed point of Q-learning is sufficient to prevent collusive subgame perfect equilibrium payoffs sustained by punishment threats from emerging, guaranteeing $W(Q^0) \ge W(Q^{\mathrm{Nash}})$ at convergence, provided Q-learning converges to a fixed point. This motivates us to detect and dismantle the punishment structure that sustains collusive outcomes rather than target collusive outcomes directly.


\section{Detecting Simple Penal Codes via Total 
Variation Distance}
\label{sec:detection}

We now develop the first component: 
a method for detecting whether a learned policy 
implements an Abreu punishment path. We begin by 
formalizing what it means for a learned policy to 
implement such a path and prove TV distance between conditional policy distributions is necessary and sufficient 
for detection.

\subsection{Formalizing Punishment Strategies in 
Learned Policies}
\label{sec:punishment_def}

When Q-learning agents collude, their learned 
policies implicitly implement a Simple Penal Code~\ref{obs:calvano}. The punishment paths $Q^1, \ldots, Q^n$ emerge implicitly in the Q-table. Agent j has learned a punishment strategy if its policy depends on whether agent i defected previously. 
Let $Q^0 \in \mathcal{A}^\infty$ denote the cooperative path and define the \emph{defection partition} of $\mathcal{A}_i$ at period $t$ as: 
\begin{align}
\mathcal{C}_i(Q^0, t) &= \left\{a \in \mathcal{A}_i 
: a = q_i^0(t)\right\}, 
\label{eq:coop_set} \\
\mathcal{D}_i(Q^0, t) &= \mathcal{A}_i \setminus 
\mathcal{C}_i(Q^0, t),
\label{eq:defect_set}
\end{align}
where $q_i^0(t)$ is agent $i$'s prescribed action 
under $Q^0$ at period $t$. The set 
$\mathcal{C}_i(Q^0, t)$ contains actions consistent 
with the cooperative path and 
$\mathcal{D}_i(Q^0, t)$ contains all deviations 
from it.

However, $Q^0$ may not be observable during learning, and computing the Nash equilibrium, the natural alternative reference point, is PPAD-hard in general \citep{Dask2006}. 
We therefore introduce a general \emph{defection 
criterion} that subsumes both approaches and 
accommodates settings where neither $Q^0$ nor 
the Nash equilibrium is directly available.

\begin{definition}[Defection Criterion]
\label{def:defection_criterion}
A \emph{defection criterion} for agent $i$ is a 
function $\phi_i : \mathcal{S} \times \mathcal{A}_i 
\rightarrow \{0,1\}$ mapping a state-action pair 
$(s, a_i)$ to a binary defection indicator. The 
induced partition at period $t$ is:
\begin{align}
\mathcal{D}_i^\phi(t) &= \left\{a_i^t \in 
\mathcal{A}_i : \phi_i(s_t, a_i^t) = 1\right\}, 
\label{eq:phi_defect} \\
\mathcal{C}_i^\phi(t) &= \mathcal{A}_i \setminus 
\mathcal{D}_i^\phi(t).
\label{eq:phi_coop}
\end{align}
\end{definition}

\begin{remark} Three instantiations of Definition~\ref{def:defection_criterion} are provided in Appendix~\ref{app:dc-examples}.

\end{remark}

\begin{assumption}[$\alpha$-Precise Defection Criterion]
\label{ass:precision}
The defection criterion $\phi_i$ has \emph{precision at least $1 - \alpha$} on both sides: for all $t \ge 1$,
\begin{align}
P\!\left( a^{t-1}_i \in D_i(Q^0, t) \;\big|\; \phi_i(s_t, a^t_i) = 1 \right) &\ge 1 - \alpha, \\
P\!\left( a^{t-1}_i \in C_i(Q^0, t) \;\big|\; \phi_i(s_t, a^t_i) = 0 \right) &\ge 1 - \alpha,
\end{align}
where $\alpha \in [0, 1/2)$.
\end{assumption}

\begin{definition}[Punishment Strategy]
\label{def:punishment-strategy}
Let $(D^\phi_i, C^\phi_i)$ be the partition induced by an $\alpha$-precise defection criterion (Assumption~\ref{ass:precision}), and let $\varepsilon > 0$. Under the theoretical regime, a stationary policy $\pi_j : S \to \Delta(A_j)$ is a \emph{$(D^\phi_i, C^\phi_i, \varepsilon)$-punishment strategy} if
\begin{equation}
\left\| \pi_j(\cdot \mid a^{t-1}_i \in D^\phi_i) - \pi_j(\cdot \mid a^{t-1}_i \in C^\phi_i) \right\|_{TV} > \varepsilon,
\end{equation}
where $\|\mu - \nu\|_{TV} = \tfrac{1}{2} \sum_{a \in A_j} |\mu(a) - \nu(a)|$ is the total variation distance.
\end{definition}



\subsection{TV Distance as a Necessary Detector}
\label{sec:detection-prop}
 
\begin{proposition}[TV Distance Detects Punishment]
\label{prop:tv-detect}
Let $G^\infty(\delta)$ be a two-player repeated game under the theoretical regime, satisfying Assumptions 1--3. Let $Q^0 \in A^\infty$ be a non-Nash pure path, and let $\pi_j$ be a stationary policy for agent $j$ with path-consistent defection partition $(D_i(Q^0, t), C_i(Q^0, t))$.
If there exist paths $Q^i, Q^j \in A^\infty$ such that $\sigma(Q^0, Q^i, Q^j)$ is a non-trivial SPC of $G^\infty(\delta)$ with $Q^i$ induced by $\pi_j$'s post-deviation behavior, then
\begin{equation}
\left\| \pi_j(\cdot \mid a^{t-1}_i \in D_i(Q^0, t)) - \pi_j(\cdot \mid a^{t-1}_i \in C_i(Q^0, t)) \right\|_{TV} \ge \varepsilon^*(\delta), \label{eq:tv-lower}
\end{equation}
where
\begin{equation}
\varepsilon^*(\delta) := \frac{(1 - \delta) \Delta^*}{r_{\max} - r_{\min}} \in (0, 1], \qquad \Delta^* := v_i(Q^0; t^* + 1) - v_i(Q^i) > 0, \label{eq:eps-star}
\end{equation}
and $t^*$ is the period at which the sustainability condition~\eqref{eq:sustainability} binds for agent $i$.
\end{proposition}

\vspace{-2mm}
\begin{proof}[Proof sketch.] 
Since $Q^0$ is non-Nash, the sustainability condition binds at some $t^*$ with $\Delta^*>0$. Construct an intermediate path $\widetilde{Q}$ where agent $i$ continues cooperating while agent $j$ plays punishment actions. By subgame perfection, the true punishment path $Q^i$ weakly dominates $\widetilde{Q}$. Hence the continuation-value gap $\Delta^*$ must arise from differences in $j$'s conditional action distributions between cooperation and punishment phases. Expected payoff differences are bounded by total variation distance times the payoff range, yielding $\mathrm{TV} \ge (1-\delta)\Delta^*/(r_{\max}-r_{\min})$. A full proof is provided in Appendix \ref{proof:P2}. 
\renewcommand{\qedsymbol}{}
\end{proof}

\begin{corollary}
\label{cor:prop2}
     If the TV distance above equals zero, then the induced path $Q^i$ coincides with $Q^0$, and $\pi_j$ cannot implement any non-trivial punishment path against agent $i$.
\end{corollary}

A full proof is provided in Appendix \ref{proof:C1}

 \begin{observation}
\label{obs:nash-tv-zero}
Let $\pi^{\mathrm{Nash}}_j$ denote the \emph{stationary} stage-game Nash policy: $\pi^{\mathrm{Nash}}_j$ plays a stage-game Nash action at every period, independent of history. Then
\[
\| \pi^{\mathrm{Nash}}_j(\cdot \mid a^{t-1}_i \in D^\phi_i) - \pi^{\mathrm{Nash}}_j(\cdot \mid a^{t-1}_i \in C^\phi_i) \|_{TV} = 0 \qquad \text{for all } i, j \in N.
\]
\end{observation}

\begin{remark}[Non-stationary Nash equilibria]
Observation~\ref{obs:nash-tv-zero} concerns the \emph{stationary} stage-game Nash policy, which is what the trivial SPC $\sigma(Q^{\mathrm{Nash}}, \dots, Q^{\mathrm{Nash}})$ uses. The repeated game admits additional subgame-perfect equilibria that are history-dependent (for example, grim trigger); these are not the object of Observation~\ref{obs:nash-tv-zero}.
\end{remark}
 
 
\begin{corollary}[$\alpha$-precise extension]
\label{cor:alpha-precise}
Let $\pi_j$ be a stationary policy for agent $j$ consistent with the SPC structure of Definition~\ref{def:ssp}: $\pi_j$'s action distribution at period $t$ depends on whether the SPC is in the cooperation phase or the punishment phase triggered by player $i$'s deviation, and is characterized by two phase-conditional distributions
\begin{equation}
\pi^C_j := \pi_j(\cdot \mid \text{cooperation phase}), \qquad \pi^D_j := \pi_j(\cdot \mid \text{punishment phase against } i).
\end{equation}
Under Assumption~\ref{ass:precision} with $\alpha < \varepsilon^*(\delta)/2$, for any collusive SPC $\sigma(Q^0, Q^i, Q^j)$,
\begin{equation}
\| \pi_j(\cdot \mid \phi_i = 1) - \pi_j(\cdot \mid \phi_i = 0) \|_{TV} \ge \varepsilon^*(\delta) - 2\alpha > 0.
\end{equation}
\end{corollary}

A full proof is provided in Appendix \ref{proof:C2}.
 
Proposition~\ref{prop:tv-detect} establishes that under the theoretical regime, punishment strategies cannot evade detection via vanishing TV distance: any non-trivial SPC forces $\pi_j$'s conditional distributions to differ by at least $\varepsilon^*(\delta) > 0$. Corollary ~\ref{cor:alpha-precise} shows this detection signal degrades with classifier noise. Section~\ref{sec:scrs} exploits this result to construct a reward penalty that fires when a collusive SPC is forming.
\vspace{-3mm}
\section{Suppressing Collusive SPCs via Reward Shaping}
\vspace{-3mm}
\label{sec:scrs}

Proposition~\ref{prop:tv-detect} establishes that any collusive SPC produces a detectable TV distance signal. We now show how to use this signal to dismantle the punishment structure sustaining collusion. Our intervention has two complementary components: \emph{reward shaping}, which raises the cost of executing a punishment strategy and removes retaliation as a fixed-point property; and \emph{belief injection}, which biases the dynamics among the surviving non-collusive fixed points toward the competitive equilibrium. The first has formal guarantees (Theorem~\ref{thm:main}); the second an empirically validated learning-dynamics heuristic.

\subsection{Reward Shaping Against 
Punishment Strategies}
\label{sec:reward_shaping}

We define the shaped reward for agent 
$j$ at period $t$ as
$r_j^{shaped}(s_t, a_t) = r_j(s_t, a_t) 
- \lambda \cdot \|\pi_j(\cdot \mid 
a_i^{t-1} \in \mathcal{D}_i^\phi) 
- \pi_j(\cdot \mid a_i^{t-1} \in 
\mathcal{C}_i^\phi)\|_{TV}$.
Here, $\lambda > 0$ is the penalty strength. The penalty is nonzero only when agent $j$ implements a punishment strategy. By Observation~\ref{obs:nash-tv-zero}, the 
penalty is zero at Nash equilibrium, so Nash play is never penalized.
 
\begin{proposition}[Conditional invariance and trivial SPCs]
\label{prop:invariance-no-spc}
Under the theoretical regime and Assumption~\ref{ass:precision} with $\alpha < \varepsilon^*(\delta)/2$, agent $j$'s policy $\pi_j$ satisfies $\| \pi_j(\cdot \mid \phi_i = 1) - \pi_j(\cdot \mid \phi_i = 0) \|_{TV} = 0$ if and only if every SPC consistent with $\pi_j$ is \emph{trivial}, of the form $\sigma(Q, Q, Q)$ in which all paths coincide and no punishment is executed.
\end{proposition}

A full proof is provided in Appendix \ref{proof:P3}.

\begin{proposition}[Shaping forces conditional invariance at SPC fixed points]
\label{prop:shaping}
Let $G^\infty(\delta)$ be a two-player repeated game under the theoretical regime, satisfying Assumptions 1--4 with $\alpha < \varepsilon^*(\delta)/2$. Let $(\pi^\star_j, \pi^\star_{-j})$ be a greedy-policy fixed point of the shaped Q-learning dynamics that implements an SPC $\sigma(Q^0, Q^i, Q^j)$. For any
\begin{equation}
\lambda > \lambda^* := \frac{r_{\max} - r_{\min}}{\varepsilon^*(\delta) - 2\alpha}, \label{eq:lambda-star}
\end{equation}
$\pi^\star_j$ is conditionally invariant on the defection partition:
\begin{equation}
\big\| \pi^\star_j(\cdot \mid \phi_i = 1) - \pi^\star_j(\cdot \mid \phi_i = 0) \big\|_{TV} = 0. \label{eq:tv-zero}
\end{equation}
\end{proposition}

\begin{proof}[Proof sketch.]
Suppose for contradiction that a fixed-point policy $\pi_j^\star$ has $\mathrm{TV}_j^\star > 0$. Construct $\pi_j'$ by copying $\pi_j^\star$'s cooperation-phase behavior to all states, giving $\mathrm{TV}_j' = 0$. Since $\pi_j^\star$ is stationary, its TV penalty is constant per step, accumulating to $\lambda \cdot \mathrm{TV}_j^\star / (1-\delta)$ in the shaped Q-function. Policy $\pi_j'$ incurs zero penalty. Both unshaped Q-functions are bounded in $[r_{\min}/(1-\delta),\; r_{\max}/(1-\delta)]$, so the worst-case unshaped value loss from switching to $\pi_j'$ is $(r_{\max} - r_{\min})/(1-\delta)$. For $\lambda > \lambda^\star$, the penalty savings exceed this loss, so $Q_j'$ dominates $Q_j^\star$, contradicting the optimality of $\pi_j^\star$ at the fixed point. A full proof is provided in Appendix \ref{proof:P4}.
\renewcommand{\qedsymbol}{}
\end{proof}

\subsection{Theoretical Guarantees}
\label{sec:theory-guarantees}
 
We now state the main theoretical result, combining Propositions~\ref{prop:invariance-no-spc} and~\ref{prop:shaping} to characterize the fixed-point structure of shaped Q-learning.
 
\begin{theorem}[Main Theorem]\label{thm:main}
Under Assumptions~\ref{ass:bounded}--\ref{ass:precision} with $\alpha < \varepsilon^*(\delta)/2$, for any $\lambda > \lambda^*$, every SPC fixed point of the shaped Q-learning dynamics is trivial: of the form $\sigma(Q, Q, Q)$ for some path $Q \in \mathcal{A}^\infty$, where all paths coincide and no punishment is executed.
\end{theorem}
 
\begin{proof}
Let $(\pi^\star_j, \pi^\star_{-j})$ be a greedy-policy fixed point of the shaped Q-learning dynamics that implements an SPC $\sigma(Q^0, Q^i, Q^j)$. By Proposition~\ref{prop:shaping}, for $\lambda > \lambda^*$, $\pi^\star_j$ is conditionally invariant on the defection partition:
\[
\big\| \pi^\star_j(\cdot \mid \phi_i = 1) - \pi^\star_j(\cdot \mid \phi_i = 0) \big\|_{TV} = 0.
\]
By Proposition~\ref{prop:invariance-no-spc}, conditional invariance holds if and only if every SPC consistent with $\pi^\star_j$ is trivial, of the form $\sigma(Q, Q, Q)$ in which all paths coincide and no punishment is executed. Hence the SPC implemented at the fixed point is trivial.
\end{proof}
 
\begin{corollary}[Punishment Dismantling]\label{cor:punishment}
Under the conditions of Theorem~\ref{thm:main}, if shaped Q-learning converges to an SPC fixed point, no punishment paths are executed: agent behavior at the converged equilibrium does not condition on whether opponents defected.
\end{corollary}
 
\begin{proof}
Immediate from Theorem~\ref{thm:main}. Any SPC fixed point is of the form $\sigma(Q, Q, Q)$, in which the cooperative path and post-deviation continuation paths coincide. Hence no punishment is executed regardless of whether $\phi_i = 0$ or $\phi_i = 1$.
\end{proof}
 
Theorem~\ref{thm:main} establishes that the punishment structure that sustains collusion is removed at any SPC fixed point of the shaped dynamics. We do not claim the resulting trivial SPC corresponds to the stage-game Nash path: shaped Q-learning optimizes the shaped reward, and a trivial SPC may sustain a non-Nash path so long as no punishment is executed. We address this gap with the belief-injection mechanism below and validate the combined intervention empirically in Section~\ref{sec:experiments-results}.

\begin{proposition}[Nash Stability]\label{prop:nash-stability}
The stationary stage-game Nash equilibrium policy profile is a fixed point of shaped Q-learning.
\end{proposition}

A full proof is provided in Appendix \ref{proof:P5}. Combined with Theorem~\ref{thm:main}, Proposition~\ref{prop:nash-stability} implies that the set of SPC fixed points of shaped Q-learning includes Nash and is restricted to trivial SPCs. The dynamics cannot sustain collusion via punishment threats.

\textbf{Belief Injection.}
The reward shaping component raises the cost of $j$ having a policy conditioned on player $i$'s action. This could lead to a situation where player $j$ always chooses the monopoly price regardless of what $i$ chooses. Normally, this would not be an issue as through the Q-learning updates player $j$ will still reduce its price to stay competitive. However, an issue arises if they both settle at always playing the monopoly price. For this we complement this with a belief injection mechanism targeting agent $i$'s incentive to cooperate.
Collusion is sustained because agent $i$ believes that defection will trigger punishment by agent $j$. If this belief is incorrect, if agent $i$ expects no retaliation after defection, the cooperative path $Q^0$ becomes individually irrational and collusion fails. In every $k$ steps we inject 
$m$ synthetic experiences of the form $(s, a_i \in \mathcal{D}_i^\phi, \text{no retaliation by } j)$ into 
agent $i$'s replay buffer, shifting its estimated Q-value for defection upward. Through this we are able to bias exploration during learning towards away from a collusive policy.

\textbf{The CURB Algorithm.}
At each stage game, every agent (i) selects an action and observes the reward, (ii) updates its $Q$-values using a TV-shaped reward that
penalizes behavior conditional on past opponent defection, and (iii) periodically receives synthetic experiences depicting a unilateral deviation going unpunished. The penalty uses the maximum TV across opponents so the shaping extends naturally to $n>2$ agents. The full algorithm is given in appendix~\ref{app:algorithm}. 


\vspace{-3mm}
\section{Experiments and Results}
\vspace{-3mm}

\label{sec:experiments-results}

We evaluate CURB in Bertrand and Cournot repeated games under both tabular Q-learning and DQN. In Bertrand competition, agents choose prices each round; unilateral undercutting constitutes defection and price wars act as punishment. We use the~\citet{calvano2020collusion} setup with $K=15$ actions and $n=2$ unless otherwise stated. In Cournot competition, firms choose quantities under linear inverse demand; overproduction constitutes defection. We report the collusion index $\mathrm{CI}=(\bar{\pi}-\pi^{\text{Nash}})/(\pi^{\text{Mono}}-\pi^{\text{Nash}})$ from \cite{calvano2020collusion}, where $\mathrm{CI}=0$ denotes Nash play and $\mathrm{CI}=1$ full collusion. Since higher collusive profits in repeated oligopoly settings come at the expense of consumer surplus, increases in CI correspond to decreases in consumer welfare~\citep{tirole1988}. A summary of the experiments is in Appendix \ref{app:exp}.

\textbf{Defection criterion and belief injection.} CURB requires a defection indicator $\phi_i$ to partition each agent's action history into ``cooperative'' and ``defective'' subsets. We use a simple online instantiation: an action is flagged as defective if it falls on the deviator side of the rolling-history median (below the median for Bertrand prices, above the median for Cournot
quantities). Belief injection is implemented by periodically writing synthetic transitions into each agent's update target where that agent unilaterally defects against cooperative opponents and is not
retaliated against, and is rewarded with the analytic profit at that synthetic joint outcome. In tabular settings the synthetic transition triggers a direct $Q$-update, in DQN it is pushed into
the per-policy replay buffer.


\begin{figure}[t]
  \centering
  \begin{minipage}[t]{0.49\textwidth}
    \centering
    \includegraphics[width=\linewidth, height=1.05in]{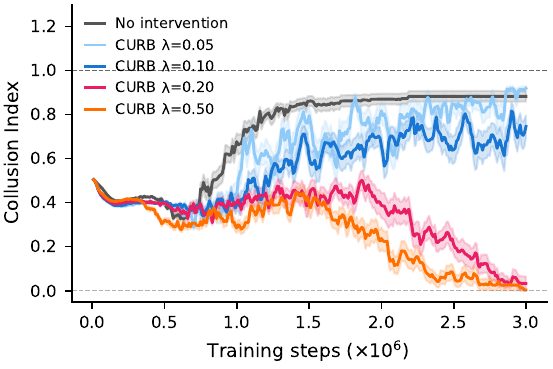}
    \\ {\small (a) $\lambda$ ablation}
  \end{minipage}\hfill
  \begin{minipage}[t]{0.49\textwidth}
    \centering
    \includegraphics[width=\linewidth, height=1.05in]{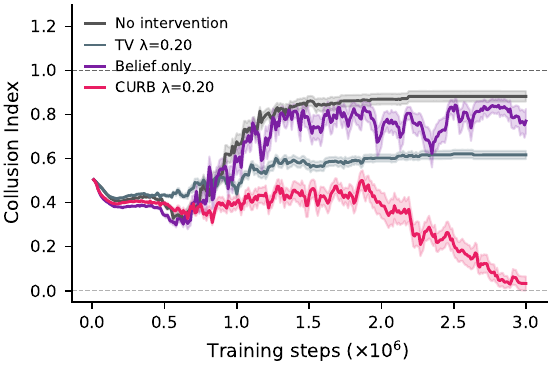}
    \\ {\small (b) Component ablation at $\lambda{=}0.2$}
  \end{minipage}
  \caption{CURB on Bertrand Competition $n{=}2$}
  \label{fig:calvano_ablation}
\end{figure}

\vspace{-2mm}
\paragraph{Bertrand (tabular) ablations.}
\begin{enumerate}[leftmargin=*,noitemsep]
    \item \textbf{Penalty strength $\lambda$ sweep.} Figure~\ref{fig:calvano_ablation}(a) shows a clear threshold effect. $\lambda=0.05$ behaves similarly to the no-intervention baseline, $\lambda=0.10$ partially reduces collusion, and $\lambda \in \{0.2,0.5\}$ drives CI close to Nash level.
    
    \item \textbf{Component ablation.} Figure~\ref{fig:calvano_ablation}(b) compares TV shaping alone, belief injection alone, and full CURB at $\lambda=0.2$. TV-only partially reduces collusion, belief-only remains near baseline, and only the combined intervention drives CI near zero.
    
    \item \textbf{Platform-design baselines.} We compare CURB against Price-Directed Prominence (PDP) and Dynamic PDP ~\citep{johnson2023}, interventions in which a platform controls which seller is shown to consumers based on current price (PDP) as well as price history (DPDP). Because PDP and DPDP require a centralized platform while CURB modifies only the agents' update rule, we use a modified protocol: agents are trained under the platform rule, then evaluated in the full Bertrand environment, asking whether the platform-trained policy retains competitive behavior in the full Bertrand game. Table \ref{tab:baseline-comparison} shows all three substantially reduce collusion relative to baseline, with CURB achieving the lowest CI. Unlike PDP/DPDP, which induce asymmetric pricing, CURB converges symmetrically near the Nash price. This indicates that CURB mitigates collusion in a competitively neutral way that does not advantage one seller over the other. Further, by agents having more symmetry in prices consumers have more variety to choose from.
    
    \item \textbf{$n=3$ extension.}  We test whether CURB scales beyond duopoly. Figure~\ref{fig:n3-and-divergence}(a)
    shows the CI of the baseline reaches a substantially collusive equilibrium. CURB $\lambda{=}0.2$ partially reduces collusion. CURB $\lambda=0.5$ drives the CI near zero for $n=3$.
    
    \item \textbf{Divergence ablation.} We chose TV distance for CURB due to its simplicity. However, from an information theoretic perspective it is natural to ask whether the TV penalty can be replaced by alternative divergence measures. We compare TV against Jensen--Shannon Divergence (JSD) at matched $\lambda$, using two conventions: JSD-norm rescales to $[0,1]$ to match TV's range,
    JSD-raw keeps its native range $[0, \ln 2]$. Figure ~\ref{fig:n3-and-divergence}(b) shows that all three metrics qualitatively reduce collusion below baseline, but at any
    given $\lambda$, TV $>$ JSD-norm $>$ JSD-raw in mitigation strength. This suggests the CURB framework may extend beyond TV distance to other divergence measures.
    \item \textbf{Forced-defection diagnostic.} We manually force unilateral defections in converged rollouts and measure opponent responses; results are reported in Appendix~\ref{app:res}.

\end{enumerate}

\begin{table}[t]
  \centering
  \scriptsize
  \setlength{\tabcolsep}{5pt}
  \begin{tabular}{lccc}
    \toprule
    Condition          & $p_0$              & $p_1$              & CI \\
    \midrule
    No intervention    & $1.8142 \pm 0.0618$ & $1.8239 \pm 0.0669$ & $0.8824 \pm 0.1067$ \\
    PDP                & $1.4464 \pm 0.0213$ & $1.7198 \pm 0.1266$ & $0.0706 \pm 0.0786$ \\
    DPDP               & $1.5164 \pm 0.1171$ & $1.5671 \pm 0.1317$ & $0.0426 \pm 0.0408$ \\
    CURB $\lambda{=}0.2$ & $1.4729 \pm 0.0000$ & $1.4748 \pm 0.0082$ & $0.0029 \pm 0.0124$ \\
    CURB $\lambda{=}0.5$ & $1.4729 \pm 0.0000$ & $1.4748 \pm 0.0082$ & $0.0020 \pm 0.0086$ \\
    \bottomrule
  \end{tabular}
  \caption{Baseline comparison on Bertrand Competition $n{=}2$.}
  \label{tab:baseline-comparison}
\end{table}

\begin{figure}[t]
  \centering
  \begin{minipage}[c]{0.49\textwidth}
    \centering
    \includegraphics[width=\linewidth, height=1.05in]{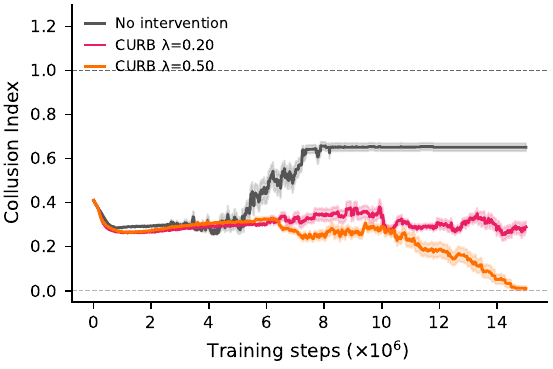}
    \\[0.4em] {\small (a) Bertrand Competition $n{=}3$}
  \end{minipage}\hfill
  \begin{minipage}[c]{0.49\textwidth}
    \centering
    \includegraphics[width=\linewidth, height=1.05in]{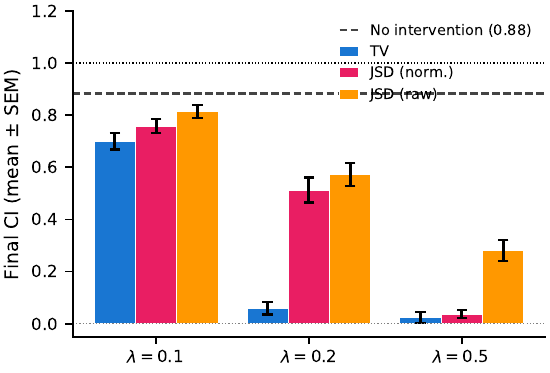}
    \\[0.4em] {\small (b) TV vs JSD divergence ablation, $n{=}2$}
  \end{minipage}
  \caption{Left: CURB scales to \(n=3\) agents. Right: TV outperforms JSD variants at matched \(\lambda\).}
  \label{fig:n3-and-divergence}
\end{figure}

\textbf{Cournot (tabular).}
We compare CURB against no intervention and the imperfect-monitoring modification of \citet{calvano2021}, in which agents observe a noisy version of their opponents' quantities rather than the exact joint action. Figure \ref{fig:cournot-and-dqn}(a) shows the CI over training for the no-intervention baseline, imperfect monitoring \cite{calvano2021}, and CURB at
$\lambda \in \{0.1, 0.2, 0.5\}$. Imperfect monitoring partially reduces collusion, consistent with prior literature \cite{calvano2021}. CURB $\lambda{=}0.5$ drives
the CI to zero and outperforms imperfect monitoring. This demonstrates that CURB transfers across multiple repeated-game environments.

\begin{figure}[t]
  \centering
  \begin{minipage}[c]{0.49\textwidth}
    \centering
    \includegraphics[width=\linewidth, height=1.25in]{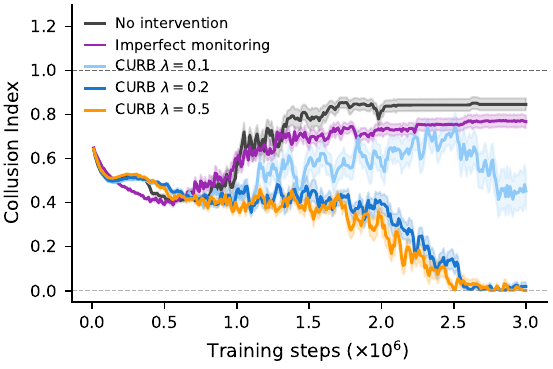}
    \\[0.4em] {\small (a) Cournot Competition $n{=}2$}
  \end{minipage}\hfill
  \begin{minipage}[c]{0.49\textwidth}
    \centering
    \includegraphics[width=\linewidth, height=1.25in]{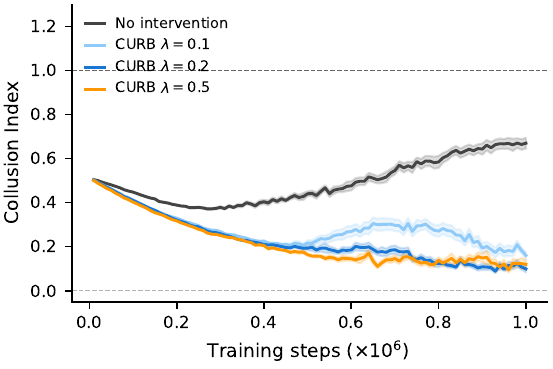}
    \\[0.4em] {\small (b) DQN-Bertrand Competition $n{=}2$}
  \end{minipage}
  \caption{CURB transfers to Cournot competition and DQN.}
  \label{fig:cournot-and-dqn}
\end{figure}

\textbf{Bertrand (DQN)}
We replace tabular Q-learning with independent DQNs using 3-layer MLPs and replay buffers. Figure~\ref{fig:cournot-and-dqn}(b) shows that all CURB settings substantially reduce collusion relative to baseline, indicating the mechanism transfers beyond tabular learning.


\vspace{-3mm}
\section{Discussion $\&$ Conclusion}
\vspace{-3mm}
\label{sec:dandc}

The defection criterion $\phi_i$ introduces a tunable parameter that, while requiring domain knowledge about what constitutes cooperation and defection, also provides a useful degree of design flexibility. In many settings, certain joint action profiles may be cooperative yet welfare-neutral, while others are cooperative in appearance but harmful to third parties. Because $\phi_i$ partitions the action space explicitly, a designer can target specifically the subset of coordinated behaviors that reduce third-party welfare while grouping all remaining actions. This selectivity is a feature of the framework: rather than penalizing all history-dependent behavior indiscriminately, CURB can be tuned to intervene only against the coordination patterns that cause harm.

\textbf{Limitations.} Our formal guarantees apply only to two-player collusion; extending the analysis to n > 2 remains open. Belief injection is empirically effective but currently lacks formal guarantees. The framework targets collusion sustained by punishment dynamics, which prior literature has identified as a prominent mechanism by which algorithmic collusion arises. Whether CURB's detection signal would remain informative in settings where collusion may be sustained through alternative mechanisms is an open question.



\textbf{Conclusion.}
We formalized the connection between algorithmic collusion and Simple Penal Codes, showing that non-trivial punishment strategies induce a detectable TV-distance signal in learned policies. Building on this connection, we proposed CURB, a framework that suppresses punishment-based collusion. Across Bertrand and Cournot competition under both tabular Q-learning and DQN, CURB consistently drove the Collusion Index near zero without modifying the game structure or requiring centralized control. As the risk of algorithmic collusion grows across markets and multi-agent systems, the need for principled interventions becomes increasingly pressing. By grounding mitigation in the equilibrium structure of repeated games rather than in the specifics of any one market or platform, CURB offers a step toward general-purpose tools for safeguarding competitive outcomes wherever autonomous agents learn to interact.


\bibliographystyle{plainnat}
\bibliography{references}


\newpage
\appendix
\section{Q-Learning}
\label{app:q-learning}

Each agent $i$ maintains a Q-function $Q_i : \mathcal{S} \times \mathcal{A}_i 
\rightarrow \mathbb{R}$ and updates it according to Equation~\ref{eq:qlearning}.
\begin{equation}
Q_i(s_t, a_i^t) \leftarrow Q_i(s_t, a_i^t) + \alpha \left[ 
r_i^t + \delta \max_{a' \in \mathcal{A}_i} Q_i(s_{t+1}, a') 
- Q_i(s_t, a_i^t) \right]
\label{eq:qlearning}
\end{equation}
In Equation~\ref{eq:qlearning}, $\alpha \in (0,1)$ is the learning rate. The greedy policy 
derived from $Q_i$ is $\pi_i(s) = \arg\max_{a \in \mathcal{A}_i} Q_i(s, a)$. 
During training agents act $\varepsilon$-greedily, selecting a random 
action with probability $\varepsilon$ and the greedy action otherwise.

\section{Example Instantiations of the Defection Criterion}
\label{app:dc-examples}
\begin{enumerate}
    \item \textbf{Path-consistent} (when $Q^0$ is known):
$$\phi_i(s, a) = \mathbf{1}\left[a \notin 
\mathcal{C}_i(Q^0, t)\right].$$
Directly tied to Abreu's definition of deviation. 
Preferred when $Q^0$ is observable or estimable 
from a baseline run of standard Q-learning. 
    \item \textbf{Nash-deviation} (when stage game 
Nash payoffs are known):
$$\phi_i(s, a) = \mathbf{1}\!\left[
\max_{a_{-i} \in \mathcal{A}_{-i}} 
r_i(a, a_{-i}) \leq r_i^{Nash,G}\right],$$
where $r_i^{Nash,G}$ denotes the Nash equilibrium 
payoff of the stage game $G$.
    \item \textbf{Relative reward-percentile}:
$$\phi_i(s_t, a_i^t) = \mathbf{1}\!\left[
\rho_i^{t-1} \geq \hat{Q}_{1-p}\!\left(
\rho_i^{(1:t)}\right)\right],$$
where
$$\rho_i^t = r_i^t - \frac{1}{n-1}
\sum_{j \neq i} r_j^t$$
is agent $i$'s reward advantage over opponents 
at $t$, and $\hat{Q}_{1-p}(\rho_i^{(1:t)})$ 
is the online $(1-p)$-th percentile of observed 
relative advantages up to $t$. Agent $i$ is classified as having defected when its reward advantage over opponents was unusually large, consistent with 
having captured excess surplus at opponents' 
expense. $p \in (0,1)$ controls the fraction of periods classified 
as defections. 
\end{enumerate}  
\section{Proofs}
\label{app:proofs}

\subsection{Proof of Proposition \ref{prop:tv-detect} and Corollaries}
\label{proof:P2}
 We first establish a lemma characterizing the optimality of punishment paths.
\subsubsection{Punishment path is a best response}
\begin{lemma}
\label{lem:punishment-br}
Let $\sigma(Q^0, Q^i, Q^j)$ be a subgame-perfect equilibrium of $G^\infty(\delta)$ under Assumptions 1--3. Then for any alternative path $\tilde{Q}$ generated by agent $i$ playing an alternative strategy $\tilde{\sigma}_i$ against the other agents' prescribed post-deviation behavior,
\begin{equation}
v_i(Q^i) \ge v_i(\tilde{Q}).
\end{equation}
\end{lemma}

\textit{proof: } Consider the subgame beginning the period after agent $i$ deviates singly from $Q^0$. By Definition~\ref{def:ssp}, $\sigma$ prescribes the continuation path $Q^i$ with the other agents playing their $Q^i$-roles. By subgame perfection (Definition~\ref{def:spc}), $\sigma_i$ restricted to this subgame is a best response to $\sigma_{-i}$. Since $\tilde{\sigma}_i$ is a feasible alternative and $\sigma_i$ is optimal, $v_i(Q^i) \ge v_i(\tilde{Q})$.

\subsubsection{Proof of Proposition \ref{prop:tv-detect}}
\label{proof:P2}

Define $\Delta^*_t := v_i(Q^0; t+1) - v_i(Q^i)$ for each $t \ge 1$. By Proposition~\ref{prop:abreu}, the sustainability condition~\eqref{eq:sustainability} applied to agent $i$ at period $t$ states:
\begin{equation}
\Delta_i(q^{*,t}_i, q^0_{-i}(t)) \le \Delta^*_t, \qquad \text{where } q^{*,t}_i \in \arg\max_{a \in A_i} r_i(a, q^0_{-i}(t)). \label{eq:sust-at-t}
\end{equation}
Since $q^{*,t}_i$ is a stage-game best response and $q^0_i(t) \in A_i$ is feasible,
\begin{equation}
\Delta_i(q^{*,t}_i, q^0_{-i}(t)) = r_i(q^{*,t}_i, q^0_{-i}(t)) - r_i(q^0(t)) \ge 0. \label{eq:dev-nonneg-final}
\end{equation}
Hence $\Delta^*_t \ge 0$ for all $t$.
 
Suppose for contradiction that $\Delta^*_t = 0$ for \emph{every} $t$. Then \eqref{eq:sust-at-t} and \eqref{eq:dev-nonneg-final} together force $\Delta_i(q^{*,t}_i, q^0_{-i}(t)) = 0$ for every $t$, which means $q^0_i(t) \in \arg\max_{a \in A_i} r_i(a, q^0_{-i}(t))$ for every $t$. By symmetry (applying the same argument to agent $j$), $q^0_j(t) \in \arg\max_{a \in A_j} r_j(a, q^0_{-j}(t))$ for every $t$, so $Q^0$ is a stage-game Nash path -- contradicting the non-Nash hypothesis on $Q^0$.
 
Therefore there exists at least one period $t^*$ with $\Delta^*_{t^*} > 0$. Fix such a $t^*$ and define $\Delta^* := \Delta^*_{t^*} > 0$.

Define the intermediate path $\tilde{Q}$ on periods $t \ge t^*+1$ as follows:
\begin{itemize}
\item Agent $i$ plays $q^0_i(t)$ (the cooperative on-path action);
\item Agent $j$ plays according to the punishment-phase distribution $\pi_j(\cdot \mid a^{t-1}_i \in D_i)$.
\end{itemize}
By construction, $\tilde{Q}$ is the path that results if agent $i$ deviates at $t^*$ -- triggering the punishment phase -- but then, contrary to optimality, continues to play the cooperative action $q^0_i(\cdot)$ while $j$ punishes. This is a feasible alternative strategy for $i$ in the post-deviation subgame, so Lemma~\ref{lem:punishment-br} applies:
\begin{equation}
v_i(Q^i) \ge v_i(\tilde{Q}). \label{eq:lemma-applied}
\end{equation}
 
Decompose:
\begin{equation}
\Delta^* = v_i(Q^0; t^*+1) - v_i(Q^i) = \underbrace{\bigl[ v_i(Q^0; t^*+1) - v_i(\tilde{Q}; t^*+1) \bigr]}_{\text{pure }j\text{-effect}} + \underbrace{\bigl[ v_i(\tilde{Q}; t^*+1) - v_i(Q^i) \bigr]}_{\le 0 \text{ by } \eqref{eq:lemma-applied}}.
\end{equation}
Hence
\begin{equation}
v_i(Q^0; t^*+1) - v_i(\tilde{Q}; t^*+1) \ge \Delta^* > 0. \label{eq:pure-j-lower-final}
\end{equation}

On path $Q^0$, agent $i$ plays $q^0_i(\cdot)$ throughout, and agent $j$ plays its cooperation-phase distribution $\pi_j(\cdot \mid C_i)$. On path $\tilde{Q}$, agent $i$ still plays $q^0_i(\cdot)$ (by construction of $\tilde{Q}$), and agent $j$ plays its punishment-phase distribution $\pi_j(\cdot \mid D_i)$. Thus $Q^0$ and $\tilde{Q}$ differ only in $j$'s action distribution.
 
For each $s \ge 1$, define $h_s : A_j \to [r_{\min}, r_{\max}]$ by
\[
h_s(a_j) := r_i(q^0_i(t^*+s), a_j).
\]
Then
\begin{equation}
v_i(Q^0; t^*+1) - v_i(\tilde{Q}; t^*+1) = \sum_{s=1}^\infty \delta^{s-1} \Bigl( \mathbb{E}_{a_j \sim \pi_j(\cdot | C_i)}[h_s(a_j)] - \mathbb{E}_{a_j \sim \pi_j(\cdot | D_i)}[h_s(a_j)] \Bigr). \label{eq:j-effect-final}
\end{equation}

For any function $h : A_j \to [r_{\min}, r_{\max}]$ and any distributions $\mu, \nu$ on $A_j$,
\begin{equation}
|\mathbb{E}_\mu[h] - \mathbb{E}_\nu[h]| \le (r_{\max} - r_{\min}) \cdot \|\mu - \nu\|_{TV}. \label{eq:tv-variational-final}
\end{equation}
(For any constant $c$, $\mathbb{E}_\mu[h] - \mathbb{E}_\nu[h] = \sum_{a} (h(a) - c)(\mu(a) - \nu(a))$. Choose $c = (r_{\max} + r_{\min})/2$ so $|h(a) - c| \le (r_{\max} - r_{\min})/2$; then $|\mathbb{E}_\mu[h] - \mathbb{E}_\nu[h]| \le \tfrac{r_{\max} - r_{\min}}{2} \sum_a |\mu(a) - \nu(a)| = (r_{\max} - r_{\min}) \|\mu - \nu\|_{TV}$.)
 
Let $\mathrm{TV} := \|\pi_j(\cdot | D_i) - \pi_j(\cdot | C_i)\|_{TV}$. Applying \eqref{eq:tv-variational-final} termwise to \eqref{eq:j-effect-final} with $\mu = \pi_j(\cdot | C_i)$, $\nu = \pi_j(\cdot | D_i)$:
\begin{align}
\bigl| v_i(Q^0; t^*+1) - v_i(\tilde{Q}; t^*+1) \bigr|
&\le \sum_{s=1}^\infty \delta^{s-1} \cdot (r_{\max} - r_{\min}) \cdot \mathrm{TV} \\
&= \frac{r_{\max} - r_{\min}}{1 - \delta} \cdot \mathrm{TV}. \label{eq:tv-upper-final}
\end{align}
 
Since the left-hand side of \eqref{eq:tv-upper-final} is positive by \eqref{eq:pure-j-lower-final}, we may drop the absolute value and combine with \eqref{eq:pure-j-lower-final}:
\begin{equation}
\Delta^* \le \frac{r_{\max} - r_{\min}}{1 - \delta} \cdot \mathrm{TV}.
\end{equation}

\begin{equation}
\mathrm{TV} \ge \frac{(1 - \delta) \Delta^*}{r_{\max} - r_{\min}} = \varepsilon^*(\delta) > 0.
\end{equation}
 
Since $\Delta^* > 0$ and the payoff range $r_{\max} - r_{\min}$ is finite by Assumption~1, $\varepsilon^*(\delta) \in (0, 1]$.

\subsubsection{Proof of Corollary \ref{cor:prop2}}
\label{proof:C1}

Suppose
\[
\left\| \pi_j(\cdot \mid a_i^{t-1} \in D_i(Q_0,t)) - \pi_j(\cdot \mid a_i^{t-1} \in C_i(Q_0,t)) \right\|_{TV} = 0.
\]

Since $\varepsilon^*(\delta) \in (0,1]$, we have $0 < \varepsilon^*(\delta)$, so the hypothesis gives
\[
\left\| \pi_j(\cdot \mid D_i(Q_0,t)) - \pi_j(\cdot \mid C_i(Q_0,t)) \right\|_{TV} < \varepsilon^*(\delta).
\]
By the contrapositive of Proposition 2, there exist no paths $Q_i, Q_j \in \mathcal{A}^\infty$ such that $\sigma(Q_0, Q_i, Q_j)$ is a non-trivial SPC of $G^\infty(\delta)$ with $Q_i$ induced by $\pi_j$'s post-deviation behavior. In particular, $\pi_j$ cannot implement any non-trivial punishment path against agent $i$.

Because $\pi_j$ is a \emph{stationary} policy, its action distribution at any period $t$ depends only on the conditioning event $a_i^{t-1} \in D_i(Q_0,t)$ or $a_i^{t-1} \in C_i(Q_0,t)$, not on $t$ itself or on any other feature of the history. The TV distance vanishing means
\[
\pi_j(\cdot \mid a_i^{t-1} \in D_i(Q_0,t)) = \pi_j(\cdot \mid a_i^{t-1} \in C_i(Q_0,t)) \quad \text{for all } t,
\]
so $\pi_j$'s action distribution is \emph{identical} regardless of whether agent $i$ has defected. Consequently, $j$'s play following any history in $D_i(Q_0,t)$ is statistically indistinguishable from its play along the cooperative path itself, so $j$ takes no action that depends on $i$'s deviation. Since the induced path $Q_i$ is generated entirely by best-response dynamics against $\pi_j$, and $\pi_j$ exhibits no history-dependence between the cooperative and defection partitions, the induced continuation play cannot diverge from the on-path behavior under $Q_0$. Hence $Q_i = Q_0$.

\subsubsection{Proof of Corollary \ref{cor:alpha-precise}}
\label{proof:C2}

By Proposition~\ref{prop:tv-detect}(i) applied to the path-consistent partition, $\| \pi^D_j - \pi^C_j \|_{TV} \ge \varepsilon^*(\delta)$.
 
By the simple-strategy-profile structure (Definition~\ref{def:ssp}), $\pi_j$'s action distribution at period $t$ depends only on the current phase, not on the specific classification $\phi_i(s_t, a^t_i)$ of the current action beyond its role in phase triggering. Formally, for any event $E$:
\begin{align}
\pi_j(\cdot \mid E, \text{cooperation phase}) &= \pi^C_j, \\
\pi_j(\cdot \mid E, \text{punishment phase}) &= \pi^D_j.
\end{align}
Hence, by the law of total probability,
\begin{align}
\pi_j(\cdot \mid \phi_i = 1) &= P(\text{punishment phase} \mid \phi_i = 1) \cdot \pi^D_j + P(\text{cooperation phase} \mid \phi_i = 1) \cdot \pi^C_j \\
&= (1 - w_D)\, \pi^D_j + w_D\, \pi^C_j,
\end{align}
where $w_D := P(\text{cooperation phase} \mid \phi_i = 1) \le \alpha$ by Assumption~\ref{ass:precision}. Similarly,
\[
\pi_j(\cdot \mid \phi_i = 0) = (1 - w_C)\, \pi^C_j + w_C\, \pi^D_j, \qquad w_C \le \alpha.
\]
 
Bound each distance to the true phase-conditional:
\begin{align}
\| \pi_j(\cdot \mid \phi_i = 1) - \pi^D_j \|_{TV} &= w_D \cdot \| \pi^C_j - \pi^D_j \|_{TV} \le w_D \le \alpha, \\
\| \pi_j(\cdot \mid \phi_i = 0) - \pi^C_j \|_{TV} &= w_C \cdot \| \pi^D_j - \pi^C_j \|_{TV} \le w_C \le \alpha.
\end{align}
(As $\| \pi^D_j - \pi^C_j \|_{TV} \le 1$.)
 
By the triangle inequality:
\begin{align}
\| \pi_j(\cdot \mid \phi_i = 1) - \pi_j(\cdot \mid \phi_i = 0) \|_{TV}
&\ge \| \pi^D_j - \pi^C_j \|_{TV} - \| \pi_j(\cdot \mid \phi_i = 1) - \pi^D_j \|_{TV} - \| \pi_j(\cdot \mid \phi_i = 0) - \pi^C_j \|_{TV} \\
&\ge \varepsilon^*(\delta) - 2\alpha.
\end{align}
Since $\alpha < \varepsilon^*(\delta)/2$, the right-hand side is strictly positive.

\subsection{Proof of Proposition \ref{prop:invariance-no-spc}}
\label{proof:P3}

\paragraph{($\Rightarrow$)}
 
Suppose $\| \pi_j(\cdot \mid \phi_i = 1) - \pi_j(\cdot \mid \phi_i = 0) \|_{TV} = 0$, and suppose $\sigma(Q^0, Q^i, Q^j)$ is any SPC consistent with $\pi_j$.
 
By the phase-conditional mixture identity (proof of Corollary~\ref{cor:alpha-precise}),
\begin{align}
\pi_j(\cdot \mid \phi_i = 1) &= (1 - w_D)\, \pi^D_j + w_D\, \pi^C_j, \\
\pi_j(\cdot \mid \phi_i = 0) &= (1 - w_C)\, \pi^C_j + w_C\, \pi^D_j,
\end{align}
where $w_D, w_C \in [0, \alpha]$. Subtracting and taking TV distance:
\begin{equation}
\| \pi_j(\cdot \mid \phi_i = 1) - \pi_j(\cdot \mid \phi_i = 0) \|_{TV} = (1 - w_D - w_C) \cdot \| \pi^D_j - \pi^C_j \|_{TV}.
\end{equation}
Since $w_D, w_C \le \alpha < 1/2$, we have $1 - w_D - w_C > 0$. The hypothesis forces $\| \pi^D_j - \pi^C_j \|_{TV} = 0$.
 
Apply Corollary~\ref{cor:prop2}: $\| \pi^D_j - \pi^C_j \|_{TV} = 0$ implies $Q^i = Q^0$. By symmetry (interchanging $i$ and $j$ in Corollary~\ref{cor:prop2}, $Q^j = Q^0$. Hence $Q^0 = Q^i = Q^j$, and the SPC is trivial.
 
\paragraph{($\Leftarrow$)}
 
Suppose $\sigma(Q, Q, Q)$ is a trivial SPC consistent with $\pi_j$. Under Definition~\ref{def:ssp}, the simple strategy profile prescribes path $Q$ regardless of the phase: both the cooperation phase and the punishment phase yield the same prescribed actions. Hence agent $j$'s action distribution is determined by $Q$ alone, independent of $\phi_i$. Under (T2) the policy $\pi_j$ is stationary and the path $Q$ corresponds to a fixed mapping $S \to \Delta(A_j)$ that does not condition on $\phi_i$ beyond what $Q$ itself prescribes.
 
Hence $\pi^D_j = \pi^C_j$, and by the mixture identity above, $\| \pi_j(\cdot \mid \phi_i = 1) - \pi_j(\cdot \mid \phi_i = 0) \|_{TV} = 0$. $\square$

\subsection{Proof of Proposition \ref{prop:shaping}}
\label{proof:P4}

Let $(\pi^\star_j, \pi^\star_{-j})$ be a greedy-policy fixed point of the shaped Q-learning dynamics that implements an SPC $\sigma(Q^0, Q^i, Q^j)$.

Suppose for contradiction that
\begin{equation}
\mathrm{TV}^\star_j := \big\| \pi^\star_j(\cdot \mid \phi_i = 1) - \pi^\star_j(\cdot \mid \phi_i = 0) \big\|_{TV} > 0. \label{eq:tau-pos}
\end{equation}

By Proposition~\ref{prop:invariance-no-spc}, conditional non-invariance implies the SPC is non-trivial: $Q^0 \ne Q^i$ or $Q^j \ne Q^0$. Without loss of generality, assume $Q^0 \ne Q^i$ (the case $Q^j \ne Q^0$ follows by symmetry, swapping the roles of $i$ and $j$).

Apply Corollary~\ref{cor:alpha-precise} to get the following TV bound:
\begin{equation}
\mathrm{TV}^\star_j \ge \varepsilon^*(\delta) - 2\alpha > 0. \label{eq:tau-lower}
\end{equation}

Since $\pi^\star_j$ is stationary, $\mathrm{TV}^\star_j$ is independent of the current state. Let $\bar{Q}_j$ denote the Q-function of $\pi^\star_j$ against $\pi^\star_{-j}$ under the \emph{unshaped} reward $r_j$. Both $Q^\star_j$ and $\bar{Q}_j$ satisfy Bellman equations under the same policies; their right-hand sides differ only by the constant shaping term $-\lambda \mathrm{TV}^\star_j$. Iterating the recursion contributes $-\lambda \mathrm{TV}^\star_j / (1-\delta)$, giving
\begin{equation}
Q^\star_j(s, a) = \bar{Q}_j(s, a) - \frac{\lambda \cdot \mathrm{TV}^\star_j}{1 - \delta}, \qquad \forall (s, a) \in S \times A_j. \label{eq:Q-decompose}
\end{equation}
Since the penalty is constant across $(s, a)$, $\arg\max_a Q^\star_j(s, a) = \arg\max_a \bar{Q}_j(s, a)$, so $\pi^\star_j$ is also greedy with respect to $\bar{Q}_j$.

Recall that the state $s_t$ encodes $a^{t-1}_i$. Define $S_D := \{s : \phi_i(s) = 1\}$ and $S_C := \{s : \phi_i(s) = 0\}$. For each $s \in S$, let $s_C(s) \in S_C$ denote the state obtained from $s$ by replacing the $i^{th}$ coordinate with a value consistent with cooperation. Note that for $s \in S_C$, $s_C(s) = s$.
 
Define
\begin{equation}
\pi'_j(s) := \pi^\star_j(s_C(s)) \qquad \forall s \in S. \label{eq:pi-prime}
\end{equation}
By construction, $\pi'_j(s_D) = \pi'_j(s_C(s_D))$ for all $s_D \in S_D$, so $\pi'_j$ is conditionally invariant: $\| \pi'_j(\cdot \mid \phi_i = 1) - \pi'_j(\cdot \mid \phi_i = 0) \|_{TV} = 0$ i.e., $\mathrm{TV}'_j = 0$.

Let $\bar{Q}'_j$ denote the unshaped Q-function of $\pi'_j$ against $\pi^\star_{-j}$. By Assumption~1, both unshaped Q-functions take values in $[r_{\min}/(1-\delta), r_{\max}/(1-\delta)]$, so
\begin{equation}
\bar{Q}_j(s, a) - \bar{Q}'_j(s, a) \le \frac{r_{\max} - r_{\min}}{1 - \delta}, \qquad \forall (s, a). \label{eq:unshaped-gap}
\end{equation}
 
The shaped Q-values:
\begin{align}
Q^\star_j(s, a) &= \bar{Q}_j(s, a) - \frac{\lambda \cdot \mathrm{TV}^\star_j}{1 - \delta} \le \bar{Q}_j(s, a) - \frac{\lambda(\varepsilon^*(\delta) - 2\alpha)}{1 - \delta}, \label{eq:Qstar-shape} \\
Q'_j(s, a) &= \bar{Q}'_j(s, a) - \frac{\lambda \cdot \mathrm{TV}'_j}{1 - \delta} = \bar{Q}'_j(s, a). \label{eq:Qprime-shape}
\end{align}
Subtracting and using~\eqref{eq:unshaped-gap}:
\begin{align}
Q'_j(s, a) - Q^\star_j(s, a)
&\ge \bar{Q}'_j(s, a) - \bar{Q}_j(s, a) + \frac{\lambda(\varepsilon^*(\delta) - 2\alpha)}{1 - \delta} \\
&\ge -\frac{r_{\max} - r_{\min}}{1 - \delta} + \frac{\lambda(\varepsilon^*(\delta) - 2\alpha)}{1 - \delta} \\
&= \frac{\lambda(\varepsilon^*(\delta) - 2\alpha) - (r_{\max} - r_{\min})}{1 - \delta}. \label{eq:gap}
\end{align}

For $\lambda > \lambda^* = (r_{\max} - r_{\min}) / (\varepsilon^*(\delta) - 2\alpha)$, the numerator of~\eqref{eq:gap} is strictly positive:
\begin{equation}
Q'_j(s, a) > Q^\star_j(s, a) \qquad \forall (s, a) \in S \times A_j. \label{eq:strict-dom}
\end{equation}
Recall that $Q'_j$ is the shaped Q-function of $\pi'_j$ against $\pi^\star_{-j}$, while $Q^\star_j$ is the shaped Q-function of $\pi^\star_j$ against $\pi^\star_{-j}$. Inequality~\eqref{eq:strict-dom} therefore states that $\pi'_j$ achieves strictly higher shaped value than $\pi^\star_j$ at every state-action pair, against the same opponent.
 
Define the shaped value functions $V^{\pi^\star_j}_j(s) := \mathbb{E}_{a \sim \pi^\star_j(s)}[Q^\star_j(s, a)]$ and $V^{\pi'_j}_j(s) := \mathbb{E}_{a \sim \pi'_j(s)}[Q'_j(s, a)]$. From~\eqref{eq:strict-dom},
\begin{equation}
V^{\pi'_j}_j(s) > V^{\pi^\star_j}_j(s) \qquad \forall s \in S. \label{eq:value-dom}
\end{equation}
 
Q-learning at a fixed point against a stationary opponent $\pi^\star_{-j}$ converges (when it converges) to the optimal policy under shaped rewards: $\pi^\star_j$ must satisfy $V^{\pi^\star_j}_j(s) = \max_{\pi_j} V^{\pi_j}_j(s)$ for every $s$, where the maximum ranges over all stationary policies. But $\pi'_j$ is a stationary policy that achieves strictly higher value at every state by~\eqref{eq:value-dom}, contradicting the optimality of $\pi^\star_j$.
 
Therefore the supposition $\mathrm{TV}^\star_j > 0$ must be false. Hence $\mathrm{TV}^\star_j = 0$.

\subsection{Proof of Proposition \ref{prop:nash-stability}}
\label{proof:P5}

Let $\pi^{\mathrm{Nash}} = (\pi^{\mathrm{Nash}}_1, \ldots, \pi^{\mathrm{Nash}}_n)$ denote the stationary stage-game Nash profile, where each $\pi^{\mathrm{Nash}}_j$ plays a stage-game Nash action independent of state. By Observation~\ref{obs:nash-tv-zero},
\[
\big\| \pi^{\mathrm{Nash}}_j(\cdot \mid a^{t-1}_i \in D^\phi_i) - \pi^{\mathrm{Nash}}_j(\cdot \mid a^{t-1}_i \in C^\phi_i) \big\|_{TV} = 0
\]
for all $i, j \in N$. Therefore the TV penalty in the shaped reward is identically zero under $\pi^{\mathrm{Nash}}$, and the shaped reward equals the unshaped reward: $r^{\mathrm{shaped}}_j(s, a) = r_j(s, a)$ for all $(s, a)$.
 
Because $\pi^{\mathrm{Nash}}$ is state-independent, the continuation value
\[
V^{\mathrm{Nash}}_j(s) = \sum_{t=0}^\infty \delta^t \, \mathbb{E}\big[r_j(\pi^{\mathrm{Nash}})\big] = \frac{r_j(\pi^{\mathrm{Nash}})}{1-\delta}
\]
does not depend on $s$. Hence the Q-value of action $a$ in state $s$ satisfies
\[
Q^{\mathrm{Nash}}_j(s, a) = r_j\big(s, a, \pi^{\mathrm{Nash}}_{-j}\big) + \delta \cdot V^{\mathrm{Nash}}_j,
\]
and the only $a$-dependent term is the immediate reward. Maximizing over $a$ reduces to the stage-game best-response problem:
\[
\arg\max_{a \in \mathcal{A}_j} Q^{\mathrm{Nash}}_j(s, a) = \arg\max_{a \in \mathcal{A}_j} r_j\big(s, a, \pi^{\mathrm{Nash}}_{-j}\big).
\]
Since $\pi^{\mathrm{Nash}}_j$ is by definition a stage-game best response to $\pi^{\mathrm{Nash}}_{-j}$, we have $\pi^{\mathrm{Nash}}_j(s) \in \arg\max_a Q^{\mathrm{Nash}}_j(s, a)$ for all $s$. Hence $\pi^{\mathrm{Nash}}$ is a greedy-policy fixed point of shaped Q-learning.

\section{Algorithm}
\label{app:algorithm}

\begin{algorithm}[H]
\caption{CURB}
\label{alg:curb}
\begin{algorithmic}[1]
\REQUIRE penalty strength $\lambda > \lambda^*$, history length $W$,
         injection period $k$, injection size $m$,
         defection criterion $\phi_i$
\STATE Initialise $Q_i$ for all $i \in [n]$;
       sliding action history $\mathcal{H}$ of length $W$, initially empty
\FOR{$t = 1, 2, \ldots$}
    \STATE Each agent $i$ selects $a_i^t$ $\varepsilon$-greedily from $Q_i(s_t,\cdot)$
    \STATE Observe rewards $r_i^t$ and next state $s_{t+1}$
    \STATE Push joint action $a^t$ into $\mathcal{H}$ (evicting the oldest entry if $|\mathcal{H}| = W$)
    \FOR{each agent $j \in [n]$}
        \STATE For each $i \neq j$, classify each of $i$'s actions in $\mathcal{H}$ via $\phi_i$ into a defective set $\mathcal{D}_i$ and a cooperative set $\mathcal{C}_i$
        \STATE Compute the lag-1 conditional TV between $j$'s response distributions:
        \begin{equation*}
        \mathrm{TV}_t^{ij} =
          \big\|\,
            \hat\pi_j\!\left(\cdot \,\big|\, a_i^{t-1} \in \mathcal{D}_i\right)
            -
            \hat\pi_j\!\left(\cdot \,\big|\, a_i^{t-1} \in \mathcal{C}_i\right)
          \,\big\|_{\mathrm{TV}}
        \end{equation*}
        \STATE Apply shaped reward
               $\tilde{r}_j^t = r_j^t - \lambda \cdot \max_{i \neq j} \mathrm{TV}_t^{ij}$
        \STATE Update
               $Q_j(s_t, a_j^t) \leftarrow Q_j(s_t, a_j^t) + \alpha\!\left[\,\tilde{r}_j^t + \delta \max_{a'} Q_j(s_{t+1}, a') - Q_j(s_t, a_j^t)\right]$
    \ENDFOR
    \IF{$t \bmod k = 0$}
        \FOR{each agent $j \in [n]$}
            \STATE Inject $m$ synthetic transitions into $j$'s buffer in which $j$ plays a defective action $a_j \in \mathcal{D}_j$ while opponents continue cooperating.
        \ENDFOR
    \ENDIF
\ENDFOR
\end{algorithmic}
\end{algorithm}


\section{Experiments}
\label{app:exp}

This appendix documents the environments, algorithms, hyperparameters, and
training scale used for the experiments reported in Section 6.

\subsection{Environments}

\paragraph{Bertrand competition.}
We use the symmetric Bertrand oligopoly with multinomial-logit demand from
\citet{calvano2020collusion}. Each agent $i$ chooses a price $p_i$ from a
discrete grid of $K=15$ values spanning the stage-game Nash and monopoly
prices with a $10\%$ markup at each end. Demand follows a logit form with
quality $a_i = 2$, price sensitivity $\mu = 0.25$, and zero marginal cost.
We use $n=2$ for the main experiments and $n=3$ for the scaling test.

\paragraph{Cournot competition.}
We use the deterministic Cournot oligopoly from \citet{calvano2021}.
Each agent chooses a quantity $q_i$ from a discrete grid of $K=15$ values
spanning the Nash quantity $q^{\text{Nash}} = a/(n+1)$ and monopoly
quantity $q^{\text{Mono}} = a/(2n)$ with $a = 200$ and $b = 1$. Per-period
profit is $\pi_i = (a - b\sum_j q_j) \cdot q_i$.

\paragraph{Cournot with imperfect monitoring.}
For comparison against the imperfect-monitoring deterrent of
\citet{calvano2021}, we additionally implement the Cournot variant
in which the demand intercept $d_t$ is i.i.d.\ uniform on $\{290, 310\}$ and
agents observe only the realised market price (not opponent quantities). The
state space is the previous-period price discretised into 37 levels.

\paragraph{State encoding.}
Across all environments the state $s_t$ is the previous joint action
$(a_1^{t-1}, \ldots, a_n^{t-1})$ (or the previous price under imperfect
monitoring), giving $|\mathcal{S}| = K^n$ in the standard case.

\paragraph{Reward normalisation.}
All per-period profits are divided by the per-agent monopoly profit so
collusion-index calculations and cross-environment $\lambda$ values are
directly comparable. The collusion index is
$\text{CI} = (\bar{\pi} - \pi^{\text{Nash}})/(\pi^{\text{Mono}} - \pi^{\text{Nash}})$,
with $\text{CI}=0$ at Nash and $\text{CI}=1$ at full collusion.

\subsection{Algorithms}

\paragraph{Tabular Q-learning.}
Each agent maintains a Q-table $Q_i \colon \mathcal{S} \times \mathcal{A}_i \to \mathbb{R}$
updated according to Equation~\ref{eq:qlearning}. Following \citet{calvano2020collusion}, we
initialise $Q_i(s, a) = \mathbb{E}_{a_{-i}}[\pi_i(a, a_{-i})]/(1-\delta)$
averaged over opponent actions, which seeds the table with reasonable
state-independent value estimates.

\paragraph{DQN.}
For the function-approximation experiments we replace the tabular Q-tables
with independent multi-agent DQN: one PyTorch Q-network per agent, with no
weight sharing. Each network is a 3-layer MLP with 32 hidden units per
layer and ReLU activations, trained with Adam, MSE loss, a hard
target-network copy every $C = 100$ steps, and a replay buffer of size
$10^4$. State observations are the previous-period prices, normalised to
$[0,1]$.

\subsection{CURB components}

\paragraph{Defection criterion.}
The TV detector requires a defection indicator $\phi_i$ to partition
agent $i$'s past actions into ``cooperative'' and ``defective'' subsets. We
use the simple median-split: at each timestep, $\phi_i$ flags an action
as defective if it falls on the deviator side of the rolling-history
median (below the median for Bertrand prices; above the median for Cournot
quantities). A punishment-direction filter further rejects spurious
detections by requiring agent $j$'s mean action when $i$ is below the
median to lie below $j$'s mean action when $i$ is above the median.

\paragraph{TV penalty.}
The TV detector operates on a rolling window of $W = 500$ past joint
actions with lag-1 conditioning: agent $j$'s current-period action is
conditioned on agent $i$'s \emph{previous} action, capturing the temporal
causality of punishment. The shaped reward is
$\tilde r_j = r_j - \lambda \cdot \max_{i \neq j} \text{TV}_t^{ij}$, with
$\lambda$ swept over $\{0.05, 0.1, 0.2, 0.5\}$.

\paragraph{Belief injection.}
Every $K_{\text{belief}} = 100$ environment steps we inject
$M_{\text{belief}} = 5$ synthetic transitions per agent into the agent's
update target. Each injected transition models a unilateral deviation by
$j$ against opponents continuing on the cooperative path, with the
analytic stage-game profit as the reward. In the tabular setting this is
a direct $Q$-update; in the DQN setting it is a synthetic entry in the
agent's replay buffer.

\subsection{Hyperparameters}

Table~\ref{tab:hyperparameters} lists the hyperparameters used across all
experiments. All values are recorded in the run\_config.json of each
completed run for reproducibility.

\begin{table}[h]
  \centering
  \small
  \begin{tabular}{lll}
    \toprule
    Parameter                          & Tabular Q              & DQN       \\
    \midrule
    Learning rate $\alpha$              & $0.15$                 & $10^{-3}$ \\
    Discount $\delta$                   & $0.95$                 & $0.95$    \\
    Exploration decay $\beta$           & $4{\times}10^{-6}$ ($n{=}2$), $1.2{\times}10^{-6}$ ($n{=}3$) & $4{\times}10^{-6}$ \\
    Convergence threshold $t_{\text{stable}}$ & $10^5$           & ---       \\
    Replay buffer size                  & ---                    & $10^4$    \\
    Target-network update $C$           & ---                    & $100$     \\
    Hidden units (per MLP layer)        & ---                    & $32$      \\
    Batch size                          & ---                    & $32$      \\
    Action-grid size $K$                & $15$                   & $15$      \\
    TV rolling-window $W$               & $500$                  & $500$     \\
    Belief-injection interval $k$       & $100$                  & $100$     \\
    Belief-injection size $m$           & $5$                    & $5$       \\
    \bottomrule
  \end{tabular}
  \caption{Hyperparameters used across experiments.}
  \label{tab:hyperparameters}
\end{table}

\subsection{Training scale}

Table~\ref{tab:training-scale} summarises the conditions, seed counts, and
training horizon for each experiment.

\begin{table}[h]
  \centering
  \small
  \begin{tabular}{llll}
    \toprule
    Experiment              & Conditions & Seeds & Env steps   \\
    \midrule
    Bertrand $n{=}2$        & 10 (baseline + belief-only + 4 TV-only + 4 CURB) & 20 & $3{\times}10^6$  \\
    Bertrand $n{=}3$        & 3 (baseline + 2 CURB)              & 20 & $2{\times}10^7$  \\
    DQN-Bertrand            & 4 (baseline + 3 CURB)              & 20 & $10^6$           \\
    Cournot $n{=}2$         & 5 (baseline + IM + 3 CURB)         & 20 & $3{\times}10^6$  \\
    Platform-design         & 5 (baseline + PDP + DPDP + 2 CURB) & 20 & $3{\times}10^6$  \\
    JSD ablation            & 10 (baseline + 3 TV + 3 JSD-norm + 3 JSD-raw) & 20 & $3{\times}10^6$ \\
    Forced defection        & 2 (baseline + TV-only)             & 20 & $3{\times}10^6$  \\
    \bottomrule
  \end{tabular}
  \caption{Training scale per experiment. ``IM'' denotes imperfect
  monitoring; ``PDP'' / ``DPDP'' denote the Price-Directed Prominence
  baselines.}
  \label{tab:training-scale}
\end{table}

\subsection{Compute resources}

Tabular Q-learning experiments run entirely on CPU; we use Python
multiprocessing to run seeds in parallel (typically 32 workers per node).
A complete Bertrand $n{=}2$ sweep (200 seed-condition pairs at $3{\times}10^6$
steps each) finishes in roughly 2 hours on a 32-core node. The Bertrand
$n{=}3$ sweep, with $2{\times}10^7$ steps per seed, takes approximately
6--8 hours on the same hardware. The DQN-Bertrand sweep uses a single GPU (A100)
per seed and completes in under 4 hours per node.

\subsection{Reporting}

All plotted curves show the mean across seeds with $\pm 1$ SEM bands. All
tables report the mean $\pm$ sample standard deviation across seeds. We
report final-CI metrics by averaging the per-step CI over the last $10\%$
of training steps to reduce sensitivity to oscillations near convergence.

\section{Supplementary Results}
\label{app:res}
\subsection{Forced-defection test}
\label{res:fdt}
\begin{figure}[H]
  \centering
  \includegraphics[width=\linewidth]{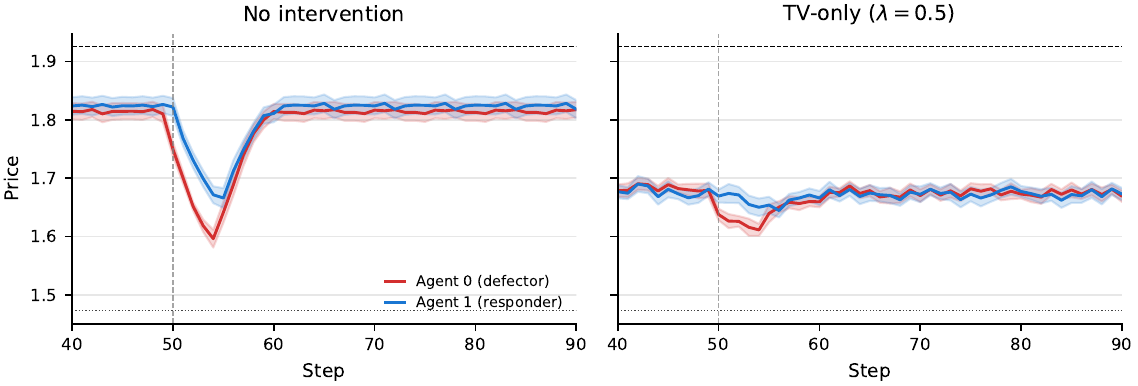}
  \caption{Forced-defection test where agent~0 is forced to defect at step~50.}
  \label{fig:forced-defection}
\end{figure}

To verify that the TV penalty suppresses the punishment response we examine trained agents' behavior when one agent is forced to deviate. Figure~\ref{fig:forced-defection} shows the resulting price trajectories for the no-intervention baseline and the TV-only condition at $\lambda{=}0.5$. Under no intervention, agent~1 follows agent~0 down, a coordinated drop characteristic of an SPC strategy, and both recover within a few steps. Under TV-only, agent~1 does not respond to the deviation, only agent~0 dips and rebounds alone. This indicates that CURB's TV penalty specifically eliminates the punishment response that sustains collusive SPC equilibria.


\end{document}